\documentclass[runningheads]{llncs}

\usepackage{cite}
\usepackage{wrapfig}
\usepackage{bm}
\usepackage{comment}
\usepackage{color}

\usepackage{amsmath}

\usepackage{amsthm}
\renewcommand\qedsymbol{$\blacksquare$}
\usepackage{amssymb}
\usepackage{amsfonts}
\usepackage{graphicx,epstopdf}
\usepackage{mathtools}
\usepackage{times}
\usepackage{color,soul}
\usepackage{color}
\usepackage{slashbox}
\usepackage{multirow}
\usepackage{todonotes}
\usepackage[noend]{algpseudocode}
\usepackage{algorithm}
\usepackage{subfigure}
\usepackage{url}
\usepackage[export]{adjustbox}
\usepackage{array}
\usepackage{stackengine}
\usepackage{booktabs}

\newcommand\NoDo{\renewcommand{\algorithmicdo}{}}
 \newcommand\NoThen{\renewcommand{\algorithmicthen}{}}

\DeclareMathOperator*{\argmin}{\arg\!\min}
\DeclareMathOperator*{\argmax}{\arg\!\max}

\DeclarePairedDelimiter\ceil{\lceil}{\rceil}
\DeclareMathAlphabet\mathbfcal{OMS}{cmsy}{b}{n}

\begin{document}

\title{Toward Adversarial Robustness by Diversity in an Ensemble of Specialized Deep Neural Networks}
\titlerunning{Toward Adv. Robustness by Diversity in an Ensemble of Specialized Deep Networks}

\author{Mahdieh Abbasi\inst{1} \and Arezoo Rajabi\inst{2} \and Christian Gagn\'e \inst{1,3} \and Rakesh B. Bobba\inst{2}}
\authorrunning{Abbasi, Rajabi, Gagn\'e, and Bobba}
%
\institute{IID, Universit\'e Laval, Qu\'ebec, Canada \and Oregon State University, Corvallis, USA \and Mila, Canada CIFAR AI Chair}

\maketitle              

\begin{abstract}
We aim at demonstrating the influence of diversity in the ensemble of CNNs on the detection of black-box adversarial instances and hardening the generation of white-box adversarial attacks. To this end, we propose an ensemble of diverse specialized CNNs along with a simple voting mechanism. The diversity in this ensemble creates a gap between the predictive confidences of adversaries and those of clean samples, making adversaries detectable. We then analyze how diversity in such an ensemble of specialists may mitigate the risk of the black-box and white-box adversarial examples. Using MNIST and CIFAR-10, we empirically verify the ability of our ensemble to detect a large portion of well-known black-box adversarial examples, which leads to a significant reduction in the risk rate of adversaries, at the expense of a small increase in the risk rate of clean samples. Moreover, we show that the success rate of generating white-box attacks by our ensemble is remarkably decreased compared to a vanilla CNN and an ensemble of vanilla CNNs, highlighting the beneficial role of diversity in the ensemble for developing more robust models.
\end{abstract}

\vspace{-1em}\section{Introduction}

Convolutional Neural Networks (CNNs) are now a common tool in many computer vision tasks with a great potential for deployement in real-world applications. Unfortunately, CNNs are strongly vulnerable to minor and imperceptible adversarial modifications of input images a.k.a. adversarial examples or adversaries. In other words, generalization performance of CNNs can be significantly dropped in the presence of adversaries. While identifying such benign-looking adversaries from their appearance is not always possible for human observers, distinguishing them from their predictive confidences by CNNs is also challenging since these networks, as uncalibrated learning models~\cite{guo2017calibration}, misclassify them with high confidence. Therefore, the lack of robustness of CNNs to adversaries can lead to significant issues in many security-sensitive real-world applications such as self-driving cars~\cite{eykholt2017robust}.

To address this issue, one line of thought, known as \emph{adversarial training}, aims at enabling CNNs to \emph{correctly classify any type of adversarial examples} by augmenting a clean training set with a set of adversaries~\cite{madry2017towards,metzen2017detecting,goodfellow2014explaining,liao2017defense,kurakin2016adversarial}. Another line of thought is to devise detectors to discriminate adversaries from their clean counterparts by training the detectors on a set of clean samples and their adversarials ones~\cite{metzen2017detecting,feinman2017detecting,grosse2017statistical,lee2018simple}. However, the performance of these approaches, by either increasing correct classification or detecting adversaries, is highly dependent on accessing a holistic set containing various types of adversarial examples. Not only generating such a large number of adversaries is computationally expensive and impossible to be made exhaustively, but adversarial training does not necessarily grant robustness to unknown or unseen adversaries~\cite{zhang2019limitation,tramer2019adversarial}. 

\begin{figure}[t!]
    \centering
    \includegraphics[width=\textwidth, trim=0cm 8cm 0cm 3cm, clip=true]{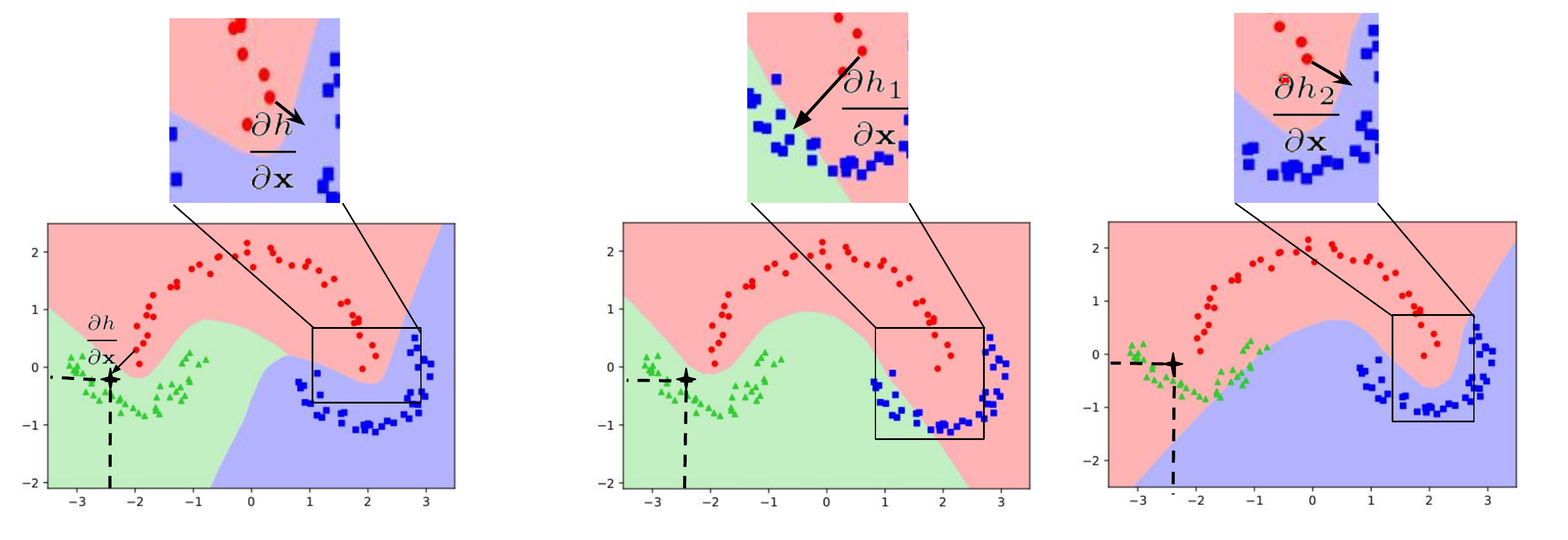}
    \vspace{-3em}
    \caption{A schematic explanation of ensemble of specialists for a 3-classes classification. On the left, a generalist ($h(.)$) trained on all 3 classes. In the middle and on the right, two specialist binary-classifiers $h_1(.)$ and $h_2(.)$ are trained on different subsets of classes, i.e. respectively (red,green) and (red, blue). A black-box attack, shown by \textbf{a black star}, which fools a generalist classifier (left), can be classified as different classes by the specialists, creating diversity in their predictions. Moreover, generation of a white-box adversarial example by the specialists can create two different fooling directions toward two unlike fooling classes. The fooling directions (in term of derivatives) are shown by black arrows in zoomed-in figures. Such different fooling directions by the specialists can harden the generation of high confidence white-box attacks (section~\ref{analysis}). \emph{Thus, by leveraging diversity in an ensemble of specialists, without the need of adversarial training, we may mitigate the risk of adversarial examples.}}
    \label{schmatic}
    \vspace{-2em}
\end{figure}
In this paper, we aim at detecting adversarial examples by predicting them with high uncertainty (low confidence) through leveraging diversity in an ensemble of CNNs, without requiring a form of adversarial training. To build a diverse ensemble, we propose forming a \emph{specialists ensemble}, where each specialist is responsible for classifying a different subset of classes. The specialists are defined so as to encourage divergent predictions in the presence of adversarial examples, while making consistent predictions for clean samples (Fig.~\ref{schmatic}). We also devise a simple voting mechanism to merge the specialists' predictions to efficiently compute the final predictions. As a result of our method, we are enforcing a gap between the predictive confidences of adversaries (i.e., low confidence predictions) and those of clean samples (i.e., high confidence predictions). By setting a threshold on the prediction confidences, we can expect to properly identify the adversaries. Interestingly, we provably show that the predictive confidence of our method in the presence of disagreement (high entropy) in the ensemble is upper-bounded by $0.5+\epsilon'$, allowing us to have a global fixed threshold (i.e., $\tau=0.5$) without requiring fine-tuning of the threshold. Moreover, we analyze our approach against the black-box and white-box attacks to demonstrate how, without adversarial training and only by diversity in the ensemble, one may design more robust CNN-based classification systems.
The contributions of our paper are as follows:
\begin{itemize}
\item We propose an ensemble of diverse specialists along with a simple and computationally efficient voting mechanism in order to predict the adversarial examples with low confidence while keeping the predictive confidence of the clean samples high, without training on any adversarial examples.
\item In the presence of high entropy (disagreement) in our ensemble, we show that the maximum predictive confidence can be upper-bounded by $0.5+\epsilon'$, allowing us to use a fixed global detection threshold of $\tau=0.5$.
\item We empirically exhibit that several types of black-box attacks can be effectively detected with our proposal due to their low predictive confidence (i.e., $\leq 0.5$). Also, we show that attack-success rate for generating white-box adversarial examples using the ensemble of specialists is considerably lower than those of a single generalist CNN and a ensemble of generalists (a.k.a pure ensemble).
\end{itemize}

\vspace{-1.5em}\section{Specialists Ensemble\label{sec:proposedmethods}}

\subsubsection{Background} For a $K$-classification problem, let us consider training set of $\{(\mathbf{x}_{i}, \mathbf{y}_i)\}_{i=1}^{N}$ with $\mathbf{x}_{i}\in\mathcal{X} $ as an input sample along with its associated ground-truth class $k$, shown by a one-hot binary vector $\mathbf{y}_{i}\in [0,1]^{K}$ with a single 1 at its $k$-th element. A CNN, denoted by $h_{\mathcal{W}}:\mathcal{X}\rightarrow [0,1]^{K}$, maps a given input to its conditional probabilities over $K$ classes. The classifier $h_\mathcal{W}(\cdot)$\footnote{For convenience, $\mathcal{W}$ is dropped from $h_\mathcal{W}(\cdot)$.} is commonly trained through a cross-entropy loss function minimization as follows:
\begin{equation}
\min_{\mathcal{W}}\frac{1}{N}\sum_{i=1}^{N}\mathcal{L}(h(\mathbf{x}_i), \mathbf{y}_i; \mathcal{W}) = -\frac{1}{N}\sum_{i=1}^{N} \log h_{k^*}(\mathbf{x}_{i}),
\label{loss-func}
\end{equation}
where $h_{k^*}(\mathbf{x}_{i})$ indicates the estimated probability of class $k^*$ corresponding to the true class of given sample $\mathbf{x}_i$. At the inference time, the threshold-based approaches like our approach define a threshold $\tau$ in order to reject the instances with lower predictive confidence than $\tau$  as an extra class ${K+1}$:  
\begin{equation}
d(\mathbf{x}|\tau) =
\begin{cases}
\argmax_k h_k(\mathbf{x}), & \text{if}~\max_k h_k(\mathbf{x}) > \tau\\
{K+1},  & \text{otherwise}
\end{cases}.
\end{equation}

\subsection{Ensemble Construction}
\label{EnsembleCreation}
We define the expertise domain of the specialists (i.e.\ the subsets of classes) by separating each class from its most likely fooled classes. We later show in Section~\ref{analysis} how separation of each class from its high likely fooling classes can promote entropy in the ensemble, which in turns leads to predicting adversaries with low confidence (high uncertainty).

To separate the most fooling classes from each other, we opt to use the fooling matrix of FGS adversarial examples $\mathbf{C}\in\mathbb{R}^{K\times K}$. This matrix reveals that the clean samples from each true class have a high tendency to being fooled toward a limited number of classes not uniformly toward all of them (Fig.\ \ref{HistAdver}(a)). The selection of FGS adversaries is two-fold; their generation is computationally inexpensive, and they are highly transferable to many other classifiers, meaning that different classifiers (e.g.\ with different structures) behave in similar manner in their presence, i.e.\ fooled to the same classes~\cite{liu2016delving,szegedy2013intriguing,charles2018geometric}.

\begin{figure}[t!]
\centering
\subfigure[ CIFAR-10 FGS fooling matrix]{\includegraphics[width = 0.5\textwidth, trim=0.3cm 8.65cm 1cm 9cm, clip=true]{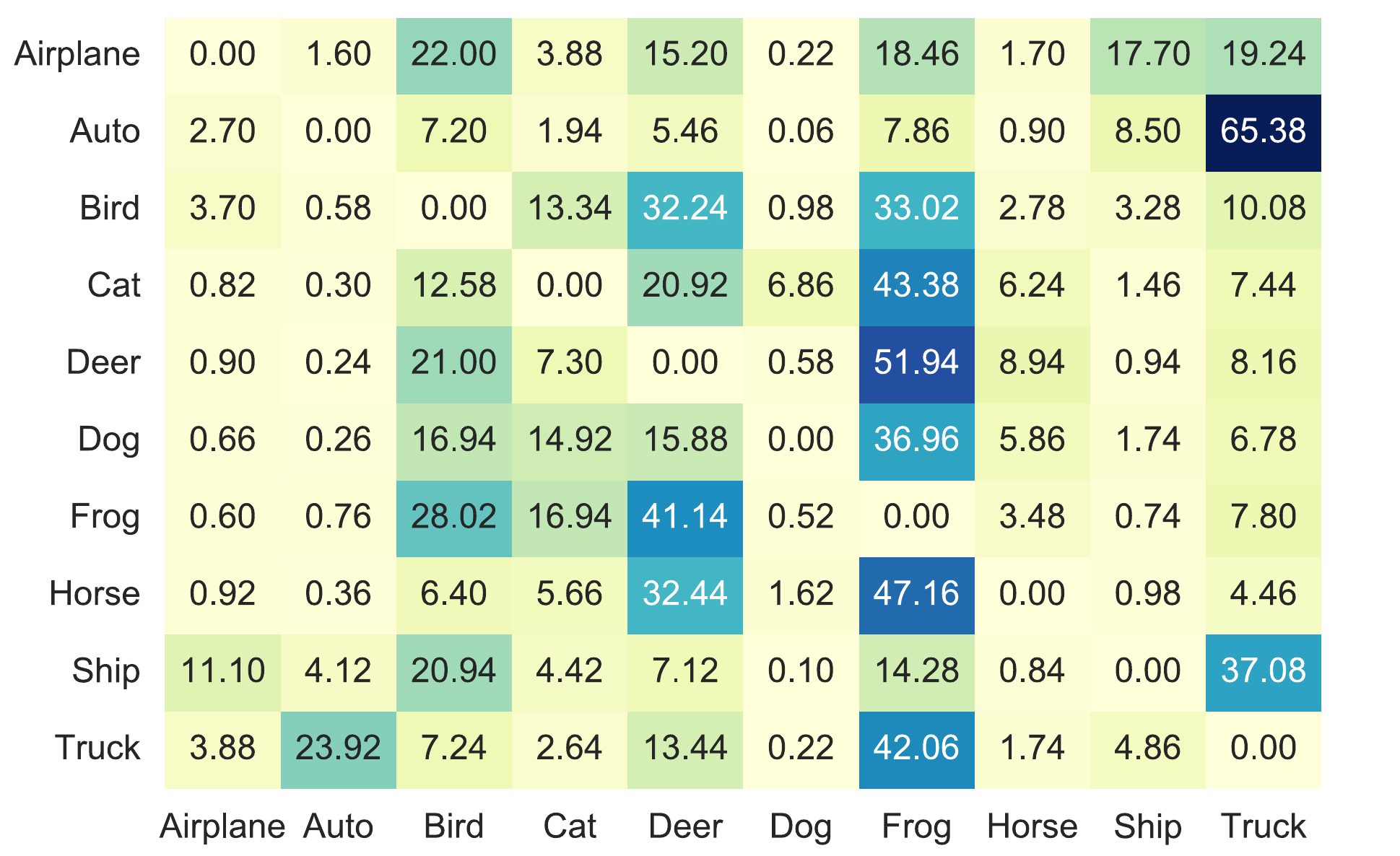}}~~\subfigure[The expertise domains of ``Airplane'' class ]{\includegraphics[width = 0.5\textwidth, trim=2cm 12cm 2cm 6cm, clip=true]{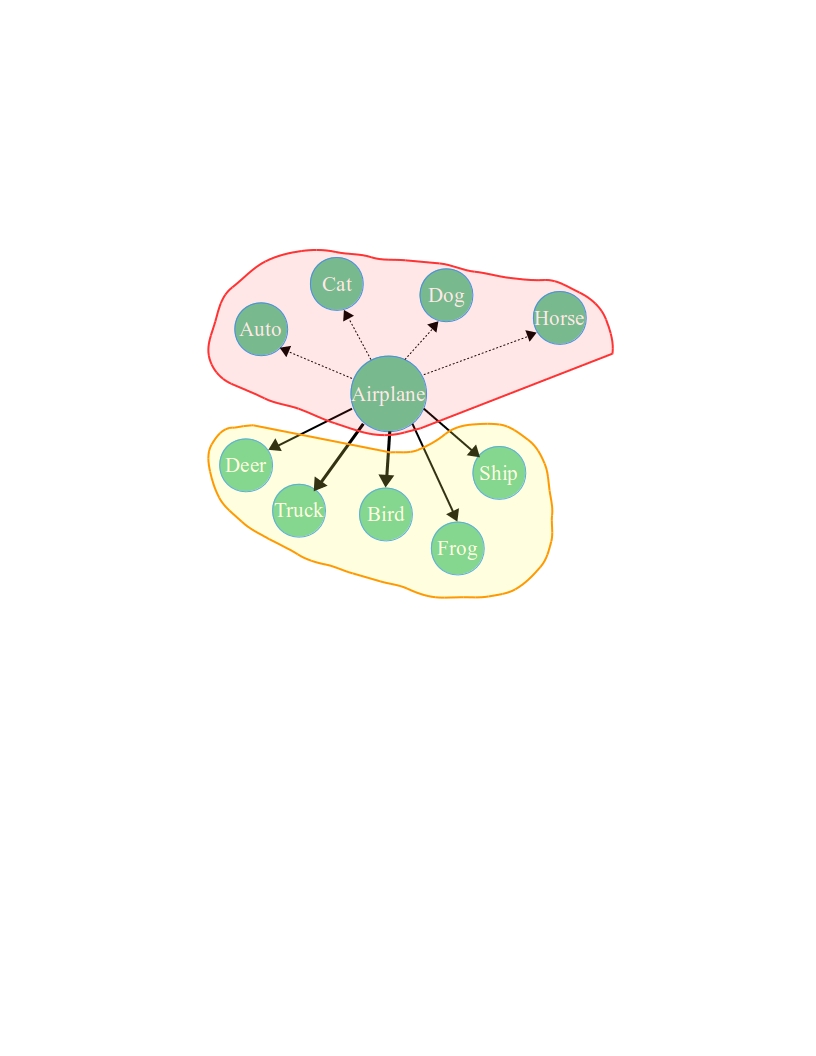}}

\label{HistAdver}
\vspace{-1em}
\caption{(a) Fooling matrix of FGS adversaries for CIFAR-10, which is computed from 5000 randomly selected FGS adversaries (500 per class). Each row shows the fooling rates (in percentage) from a true class to other classes (rows and columns are true and fooling classes, respectively). (b) An example of forming expertise domains for class ``Airplane'': its high likely fooled classes (in yellow zone) and less likely fooled classes (in red zone) are forming two expertise domains. }
\end{figure}


Using each row of the fooling matrix (i.e.\ $\mathbf{c}_i$), we define two expertise domains for $i$-th true class so as to split its high likely fooling classes from its less likely fooling classes as follows (Fig.~\ref{HistAdver}(b)):
\begin{itemize}
    \item Subset of high likely fooling classes of $i$: ~$\mathbb{U}_{i} =\cup \{j\}$~~\text{if}~~ $c_{ij}>\mu_{i}$, $j\in \{1,\dots,K\}$
    \item Subset of less likely fooling classes of $i$: $\mathbb{U}_{i+K}=\{1,\dots,K\}\setminus \mathbb{U}_i$,
\end{itemize}
where $\mu_{i}=\sum_{j=1}^{K}c_{ij}$ (average of fooling rates of $i$-th true class).
Repeating the above procedure for all $K$ classes makes $2K$ subsets (expertise domains) for a $K$ classification problem. Note that the duplicated expertise domains can be removed so as to avoid having multiple identical expertise domains (specialists). 

Afterwards, for each expertise domain, one specialist is trained in order to form an ensemble of specialist CNNs. A generalist (vanilla) CNN, which trained on the samples belonging to all classes, is also added to this ensemble. The ensemble involving $M \leq 2K+1$ members is represented by $\mathcal{H}=\{h^{1},\ldots,h^{M}\}$, where $h^{j}(\cdot)\in [0,1]^{K}$ is $j$-th individual CNN mapping a given input to conditional probability over its expert classes, i.e.\ the probability of the classes out of its expertise domain is fixed to zero.

\subsection{Voting Mechanism}
To compute the final prediction out of our ensemble for a given sample, we need to activate relevant specialists, then averaging their prediction along with that of the generalist CNN. Note that we cannot simply use the generalist CNN to activate specialists since in the presence of adversaries it can be fooled, then causing selection (activation) of the wrong specialists. In Algorithm~\ref{Votingmechanism}, we devise a simple and computationally efficient voting mechanism to activate those relevant specialists, then averaging their predictions.

Let us first introduce the following elements for each class $i$:
\begin{itemize}
    \item The actual number of votes for $i$-th class by the ensemble for a given sample $\mathbf{x}$:
$v_i(\mathbf{x}) = \sum_{j=1}^{M} \mathbb{I}\left( i = \argmax_{\{1,\dots K\}} h^j(\mathbf{x})\right)$, i.e.\ it shows the number of the members that classify $\mathbf{x}$ to $i$-th class.
    \item The maximum possible number of votes for $i$-th class is $\ceil{\frac{M}{2}} \leq K+1$. Recall that for each row, we split all $K$ classes into two expertise domains, where class $i$ is included in one of them. Considering all $K$ rows and the generalist, we end up having at maximum $K+1$ subsets that involve class $i$.

\end{itemize}

\begin{algorithm}[t]
\caption{Voting Mechanism}
\label{Votingmechanism}
\begin{algorithmic}[1]
\Require Ensemble $\mathcal{H}=\{h^{1},\ldots,h^{M}\}$, expertise domains $\mathcal{\mathbb{U}}=\{\mathbb{U}_1,\ldots,\mathbb{U}_M\}$, input $\mathbf{x}$
\Ensure Final prediction $\bar{h}(\mathbf{x})\in [0,1]^{K}$
\NoDo
\NoThen
\State {$v_k(\mathbf{x}) \gets \sum_{j=1}^{M} \mathbb{I}\left(k = \argmax_{i=1}^K h^j_i(\mathbf{x})\right),~k=1,\ldots,K \label{equation_alg}$}
\State {$k^* \gets \argmax_{k=1}^K v_k(\mathbf{x})$}

\If $v_{k^*}(\mathbf{x}) = \ceil{\frac{M}{2}}$
  \State $\mathcal{H}_{k^*} \gets \{h^i\in\mathcal{H}\,|\,k^*\in \mathbb{U}_i\}$
  \State $\bar{h}(\mathbf{x}) \gets \frac{1}{|\mathcal{H}_{k^*}|}\sum_{h^i\in\mathcal{H}_{k^*}} h^i(\mathbf{x})$
\Else
  \State $\bar{h}(\mathbf{x}) \gets \frac{1}{M}\sum_{h^i\in\mathcal{H}} h^i(\mathbf{x}) \label{othrws}$
\EndIf\\
\Return $\bar{h}(\mathbf{x})$
\end{algorithmic}

\end{algorithm}
%

As described in Algorithm~\ref{Votingmechanism}, for a given sample $\mathbf{x}$, if there is a class with its actual number of votes equal to its expected number of votes, i.e.\ $v_i(\mathbf{x})=\ceil{\frac{M}{2}}$, then it means all of the specialists, which are trained on ${i}$-th class, are simultaneously voting (classifying) for it. We call such a class \emph{a winner class}. Then, the specialists CNNs voting to the winner class are activated to compute the final prediction (lines 3--5 of Algorithm~\ref{Votingmechanism}), producing a certain prediction (with high confidence). Note that in the presence of clean samples, the relevant specialists in the ensemble are expected to do agree on the true classes since they, as strong classifiers, have high generalization performance on their expertise domains.

If no class obtains its maximum expected number of votes (i.e.\ $\nexists i,\ v_{i}(\mathbf{x}) = \ceil{\frac{M}{2}}$ ), it means that the input $\mathbf{x}$ leads the specialists to disagree on a winner class. In this situation, when no agreement exists in the ensemble, all the members should be activated to compute the final prediction (line~\ref{othrws} of Algorithm~\ref{Votingmechanism}). Averaging of the predictions by all the members leads to a final prediction with high entropy (i.e.\ low confidence). Indeed, a given sample that creates a disagreement (entropy) in the ensemble is either a hard-to-classify sample or an abnormal sample (e.g.\ adversarial examples).

Using the voting mechanism for this specialists ensemble, we can create a gap between the predictive confidences of clean samples (having high confidence) and those of adversaries (having low confidence). Finally, using a threshold $\tau$ on these predictive confidences, the unusual samples are identified and rejected. In the following, we argue that our voting mechanism enables us to set a global fixed threshold $\tau=0.5$ to perform identification of adversaries. This is unlike some threshold-based approaches~\cite{lee2018simple,bendale2016towards} that need to tune different thresholds for various datasets and their types of adversaries.

\begin{corollary} In a disagreement situation, the proposed voting mechanism makes the highest predictive confidence to be upper-bounded by $0.5+\epsilon'$ with $\epsilon'=\frac{1}{2M}$.
\label{coro1}
\end{corollary}
\begin{proof}
Consider a disagreement situation in the ensemble for a given $\mathbf{x}$, where all the members are averaged to create $\bar{h}(\mathbf{x})=\frac{1}{M}\sum_{h^{j}\in \mathcal{H}} h^{j}(\mathbf{x})$. The highest predictive confidence of $\bar{h}(\mathbf{x})$ belongs to the class that has the largest number of votes, i.e.\ $m = \max[v_{1}(\mathbf{x}),\ldots,v_{K}(\mathbf{x})]$. Let us represent these $m$ members that are voting to this class ($k$-th class) as $\mathcal{H}_k=\{h^j\in\mathcal{H}\,|\,k\in \mathbb{U}_j\}$. Since each individual CNNs in the ensemble are basically uncalibrated learners (having very high confident prediction for a class and near to zero for the remaining classes), the confidence probability of $k$-th class of those excluded members from $\mathcal{H}_k$ (those that do not vote for $k$-th class) can be negligible. Thus, their prediction can be simplified as $\bar{h}_{k}(\mathbf{x}) =\frac{1}{M}\sum_{h^{j}\in \mathcal{H}_k} h^{j}_{k}(\mathbf{x})+\frac{\epsilon}{M} \approx \frac{1}{M}\sum_{h^{j}\in \mathcal{H}_k} h^{j}_{k}(\mathbf{x})$ (the small term $\frac{\epsilon}{M}$ is discarded). Then, from the following inequality $\sum_{h^{j}\in \mathcal{H}_k} h^{j}_{k}(\mathbf{x})\leq m$, we have $\frac{1}{M}\sum_{h^{j}\in \mathcal{H}_k} h^{j}_{k}(\mathbf{x})\leq \frac{m}{M}$ (I). 

On the other hand, due to having no winner class, we know that $m < \ceil{\frac{M}{2}}$ (or $m < \frac{M}{2}+\frac{1}{2}$), such that by multiplying it by $\frac{1}{M}$ we obtain $\frac{m}{M} < \frac{1}{2}+\frac{1}{2M}$ (II). 

Finally considering (I) and (II) together, it derives $\frac{1}{M}\sum_{h^{j}\in \mathcal{H}_k} h^{j}_{k}(\mathbf{x})< 0.5+ \frac{1}{2M}$. For the ensemble with a large size, e.g.\ likewise our ensemble, the term $ \epsilon'=\frac{1}{2M}$ is small. Therefore, it shows the class with the maximum probability (having the maximum votes) can be upper-bounded by $0.5+\epsilon'$.
\qedsymbol
\end{proof}


\vspace{-1em}\section{Analysis of Specialists ensemble}
\label{analysis}
Here, we first explain how adversarial examples give rise to entropy in our ensemble, leading to their low predictive confidence (with maximum confidence of $0.5+\epsilon'$). As well, we examine the role of diversity in our ensemble, which harden the generation of white-box adversaries.



In a \textbf{black-box} attack, we assume that the attacker is not aware of our ensemble of specialists, thus generates some adversaries from a pre-trained vanilla CNN $g(\cdot)$ to mislead our underlying ensemble. Taking a pair of an input sample with its true label, i.e.\ $(\mathbf{x}, k)$, an adversary $\mathbf{x}' = \mathbf{x} + \delta$ fools the model $g$ such that $k =\argmax g(\mathbf{x})$ while $k' =\argmax g(\mathbf{x}')$ with $k'\neq k$, where $k'$ is one of those most-likely fooling classes for class $k$ (i.e.\ $k'\in\mathbb{U}_k$). Among the specialists that are expert on $k$, at least one of them does not have $k'$ in their expertise domains since we intentionally separated $k$-th class from its most-likely fooling classes when defining its expertise domains (Section~\ref{EnsembleCreation}). Formally speaking, denote those expertise domains comprising class $k$ as follows $\mathcal{U}^{k} = \{\mathbb{U}_{j}\,|\,k\in \mathbb{U}_{j}\}
$ where (I) $\mathbb{U}_{j}\neq \mathbb{U}_{i}~\forall \mathbb{U}_{i},~\mathbb{U}_{j}\in \mathcal{U}^{k}$ and (II) $k'\notin \cap~\mathcal{U}^{k}$. Therefore, regarding the fact that (I) the expertise domains comprising $k$ are different and (II) their shared classes do not contain $k'$, it is not possible that all of their corresponding specialists models are fooled simultaneously toward $k'$. In fact, these specialists may vote (classify) differently, leading to a disagreement on the fooling class $k'$. So, due to this disagreement in the ensemble with no winner class, all the ensemble's members are activated, resulting in prediction with high uncertainty (low confidence) according to corollary~\ref{coro1}. Generally speaking, if $\{\cap~\mathcal{U}^{k}\}\setminus{k}$ is a small or an empty set, harmoniously fooling the specialist models, which are expert on $k$, is harder.

In a \textbf{white-box attack}, an attacker attempts to generate adversaries to \emph{confidently} fool the ensemble, meaning the adversaries should simultaneously activate \emph{all} of the specialists that comprise the fooling class in their expertise domain. Otherwise, if at least one of these specialists is not fooled, then our voting mechanism results in adversaries with low confidence, which can then be automatically rejected using the threshold ($\tau=0.5$). In the rest we bring some justifications on the hardness of generating high confidence gradient-based attacks from the specialists ensemble.

Instead of dealing with the gradient of one network, i.e.\ $\frac{\partial {h}(\mathbf{x})}{\partial \mathbf{x}}$, the attacker should deal with the gradient of the ensemble, i.e.\ $\frac{\partial \bar{h}(\mathbf{x})}{\partial \mathbf{x}}$, where $\bar{h}(\mathbf{x})$ computed by line 5 or line 7 of Algorithm.~\ref{Votingmechanism}. Formally, to generate a gradient-based adversary from the ensemble for a given labeled clean input sample $(\mathbf{x}, \mathbf{y}=k)$, the derivative of the ensemble's loss, i.e.\ $\mathcal{L}(\bar{h}(\mathbf{x}), \mathbf{y}) = -\log \bar{h}_{k}(\mathbf{x})$, w.r.t.\ $\mathbf{x}$ is as follows:
\begin{equation}
\frac{\partial \mathcal{L} (\bar{h}(\mathbf{x}), \mathbf{y})}{\partial \mathbf{x}}
 = \frac{\partial \mathcal{L} }{\partial \bar{h}_{k}(\mathbf{x})} \frac{\partial \bar{h}_{k}(\mathbf{x})}{\partial \mathbf{x}} =  \underbrace{-\frac{1}{\bar{h}_{k}(\mathbf{x})}}_{\beta} \frac{\partial \bar{h}_{k}(\mathbf{x})}{\partial \mathbf{x}} = \beta \frac{1}{|\mathcal{H}_k|}\sum_{h^{i}\in \mathcal{H}_k}\frac{\partial h^{i}_{k}(\mathbf{x})}{\partial \mathbf{x}}.
\label{FGS-pertb}
\end{equation}
Initially $\mathcal{H}_k$ indicates the set of activated specialists voting for class $k$ (true label) plus the generalist for the given input $\mathbf{x}$.
Since the expertise domains of the activated specialists are different ($\mathcal{U}^{k} = \{\mathbb{U}_{j}\,|\,k\in \mathbb{U}_{j}\}$), most likely their derivative are diverse, i.e.\ fooling toward different classes, which in turn creates perturbations in various fooling directions (Fig~\ref{schmatic}). Adding such diverse perturbation to a clean sample may promote disagreement in the ensemble, where no winner class can be agreed upon. In this situation, when all of the members are activated, the generated adversarial sample is predicted with a low confidence, thus can be identified. For the iterative attack algorithms, e.g.\ I-FGS, the process of generating adversaries may continue using the derivative of all of the members, adding even more diverse perturbations, which in turn makes reaching to an agreement in the ensemble on a winner fooling class even more difficult.

\vspace{-1em}\section{Experimentation}\label{sec:evaluation}

\textbf{Evaluation Setting}:
Using MNIST and CIFAR-10, we investigate the performance of our method for reducing the risk rate of black-box attacks (Eq.~\ref{risk_adv_eq}) due to of their detection, and reducing the success rate of creating white-box adversaries. Two distinct CNN configurations are considered in our experimentation:  for \emph{MNIST}, a basic CNN with three convolution layers of respectively 32, 32, and 64  filters of $5\times 5$, and a final fully connected (FC) layer with 10 output neurons. Each of these convolution layers is followed by a ReLU and $3\times3$ pooling filter with stride 2. For \emph{CIFAR-10}, a VGG-style CNN (details in~\cite{simonyan2014very}) is used. For both CNNs, we use SGD with a Nesterov momentum of $0.9$, L2 regularization with its hyper-parameter set to $10^{-4}$, and dropout ($p=0.5$) for the FC layers. For the evaluation purposes, we compare our ensemble of specialists with a vanilla (naive) CNN, and a pure ensemble, which involves 5 vanilla CNNs being different by random initialization of their parameters.

\textbf{Evaluation Metrics}:
To evaluate a predictor $h(\cdot)$ that includes a rejection option, we report a risk rate $E_D|\tau$ on a clean test set $\mathcal{D}=\{(\mathbf{x}_i,\mathbf{y}_i)\}_{i=1}^N$ at a given threshold $\tau$, which computes the ratio of the (clean) samples that are \emph{correctly classified but rejected} due to their confidence less than $\tau$ and those that are \emph{misclassified but not rejected} due to a confidence value above $\tau$:
\begin{equation}
\begin{split}
E_{D}|\tau = \frac{1}{N}\sum_{i=1}^{N} \bigg( & \left(\mathbb{I}[d(\mathbf{x}_{i}|\tau)\neq K+1]~\times~\mathbb{I}[\argmax h(\mathbf{x}_i)\neq\mathbf{y}_i]\right) \\ & +~\left(
 \mathbb{I}[d(\mathbf{x}_{i}|\tau)=K+1]~\times~\mathbb{I}[\argmax h(\mathbf{x}_i)=\mathbf{y}_i]\right)\bigg).
 \end{split}
\end{equation}
In addition,  we report the risk rate $E_{A}|\tau$ on each adversaries set, i.e.\  $\mathcal{A}=\{(\mathbf{x}'_i,\mathbf{y}_i)\}_{i=1}^{N'}$ including pairs of an adversarial example $\mathbf{x}'_i$ associated by its true label, to show the percentage of misclassified adversaries that are not rejected due to their confidence value above $\tau$:
\begin{equation}
E_{A}|\tau= \frac{1}{N'}\sum_{i=1}^{N'} \left(\mathbb{I}[ d(\mathbf{x}'_{i}|\tau)\neq {K+1}]~\times~\mathbb{I}[\argmax h(\mathbf{x}'_i)\neq \mathbf{y}_i]\right).
\label{risk_adv_eq}
\end{equation}

\subsection{Empirical Results}

\textbf{Black-box attacks}:
To assess our method on different types of adversaries, we use various attack algorithms, namely FGS~\cite{goodfellow2014explaining}, TFGS~\cite{kurakin2016adversarial}, DeepFool (DF)~\cite{moosavi2015deepfool}, and CW~\cite{carlini2017towards}. To generate the black-box adversaries, we use another vanilla CNN, which is different from all its counterparts involved in the pure ensemble-- by using different random initialization of its parameters. For FGS and T-FGS algorithms we generate $2000$ adversaries with $\epsilon=0.2$ and $\epsilon=0.03$, respectively, for randomly selected clean test samples from MNIST and CIFAR-10. For CW attack, due to the high computational burden required, we generated 200 adversaries with $\kappa=40$, where larger $\kappa$ ensures generation of high confidence and highly transferable CW adversaries.

Fig.~\ref{error_threshold} presents risk rates ($E_D|\tau$) of different methods on clean test samples of MNIST (first row) and those of CIFAR-10 (second row), as well as their corresponding adversaries $E_A|\tau$, as functions of threshold ($\tau$). As it can be seen from Fig.~\ref{error_threshold}, by increasing the threshold, more adversaries can be detected (decreasing $E_A$) at the cost of increasing $E_D$, meaning an increase in the rejection of the clean samples that are correctly classified.  
\begin{figure}[ht!]
    \centering
    \subfigure[MNIST test data]{\includegraphics[width = 0.3\textwidth, trim=0cm 7cm 1cm 8cm, clip=true]{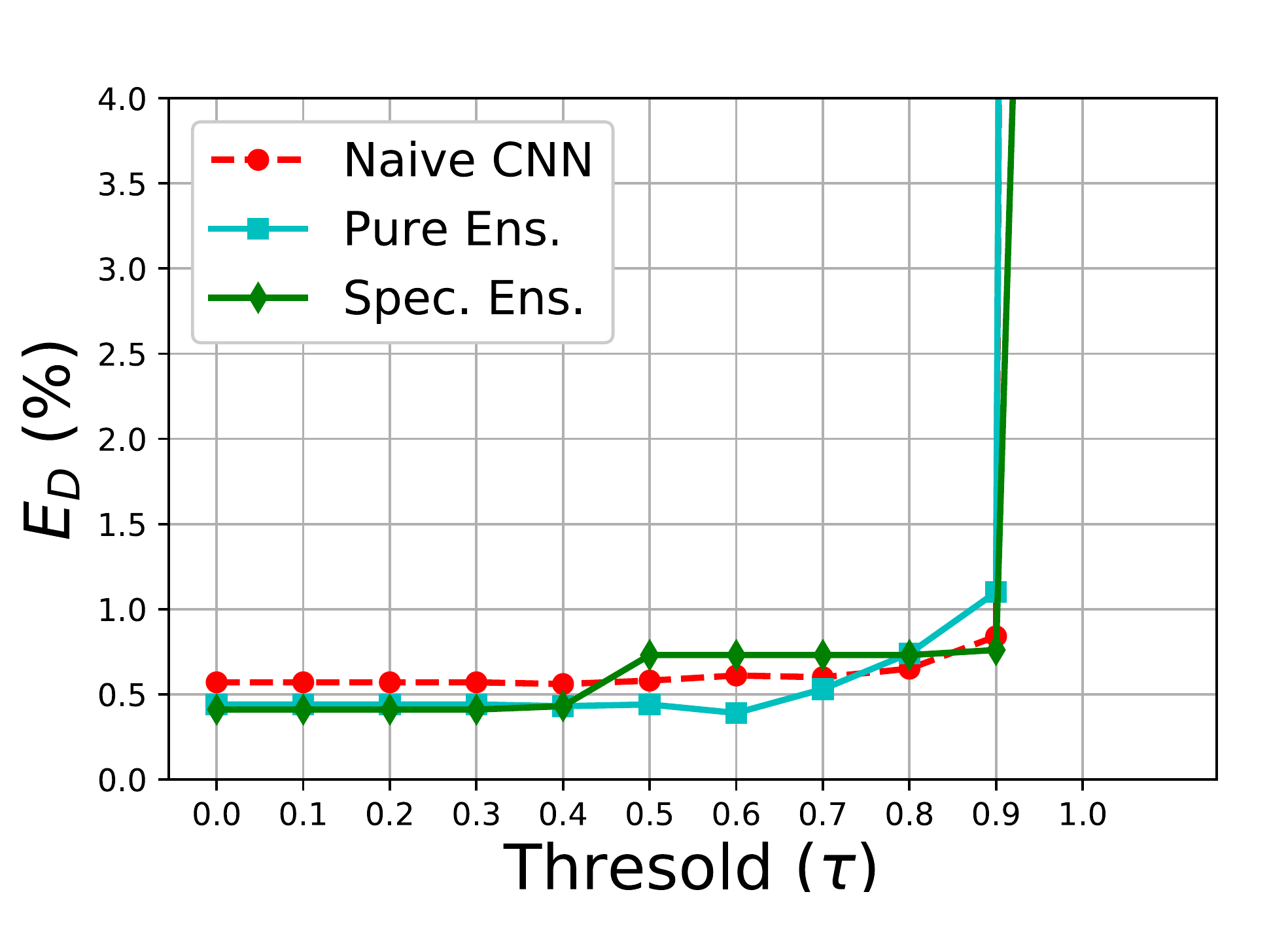}}~
    \subfigure[MNIST FGS]{\includegraphics[width = 0.3\textwidth, trim=0cm 7cm 1cm 8cm, clip=true]{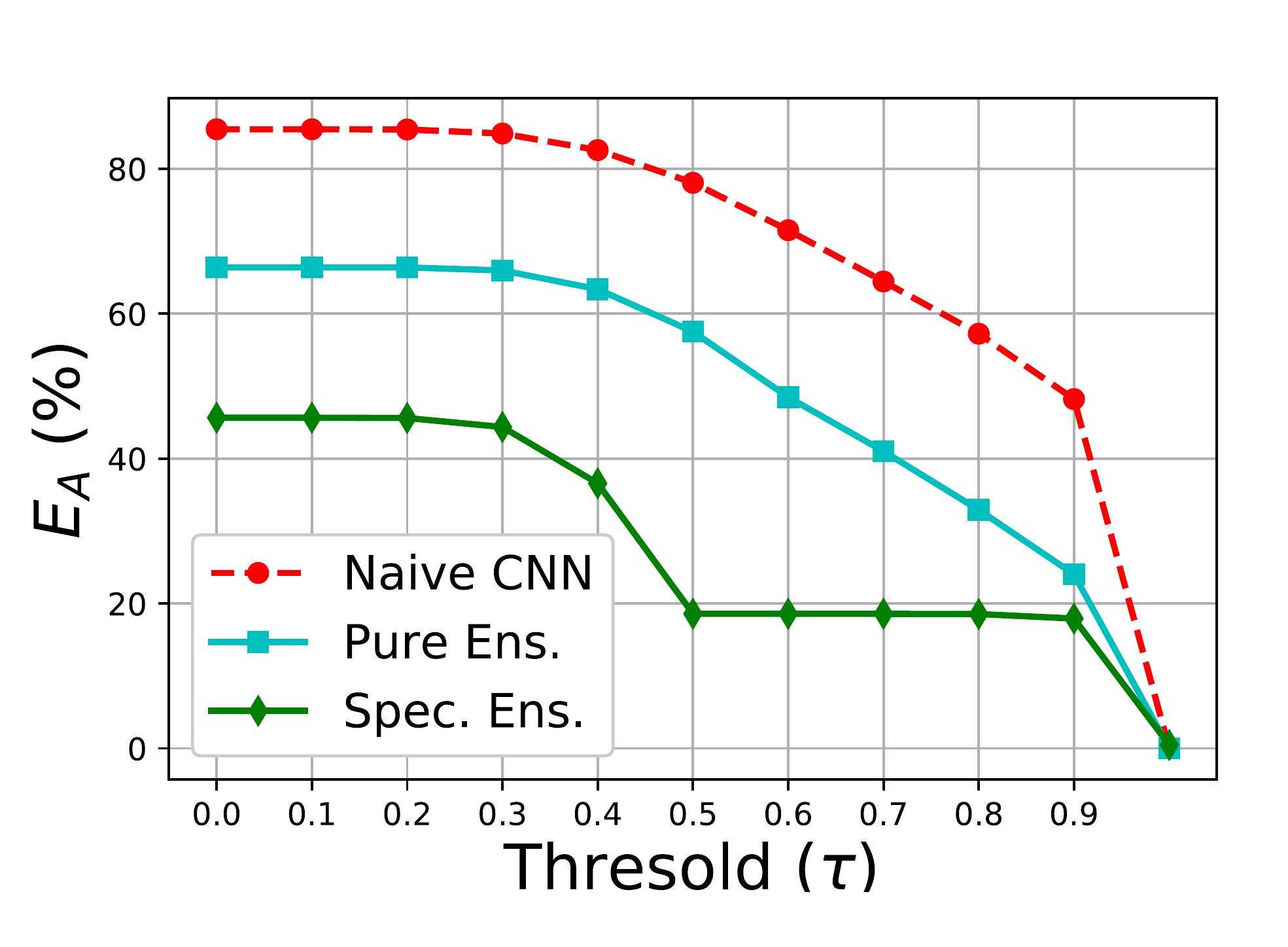}}~
    \subfigure[MNIST TFGS]{\includegraphics[width = 0.3\textwidth, trim=0cm 7cm 1cm 8cm, clip=true]{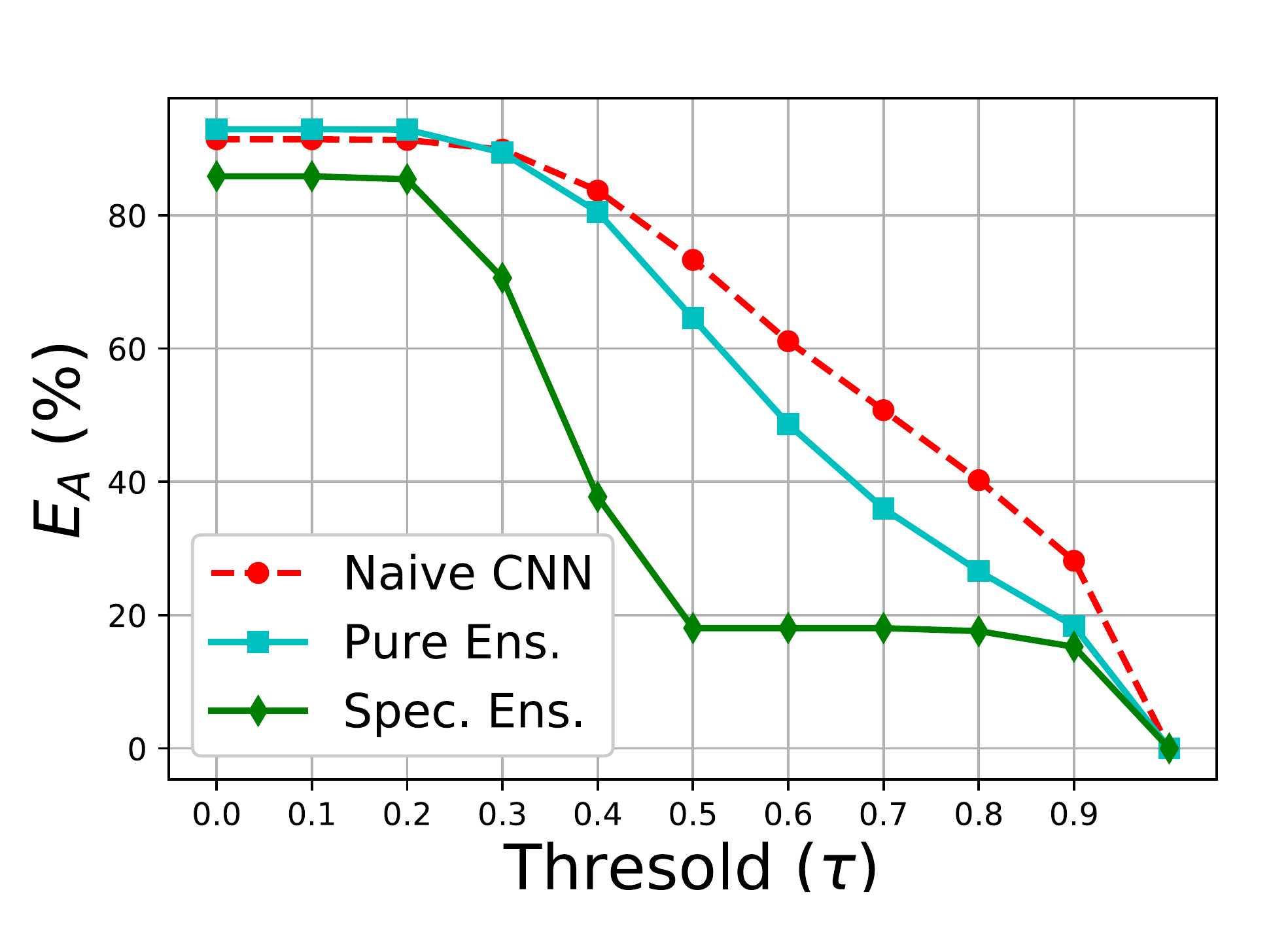}}\\
    
    \subfigure[CIFAR-10 test data]{\includegraphics[width = 0.3\textwidth, trim=0cm 7cm 1cm 8cm, clip=true]{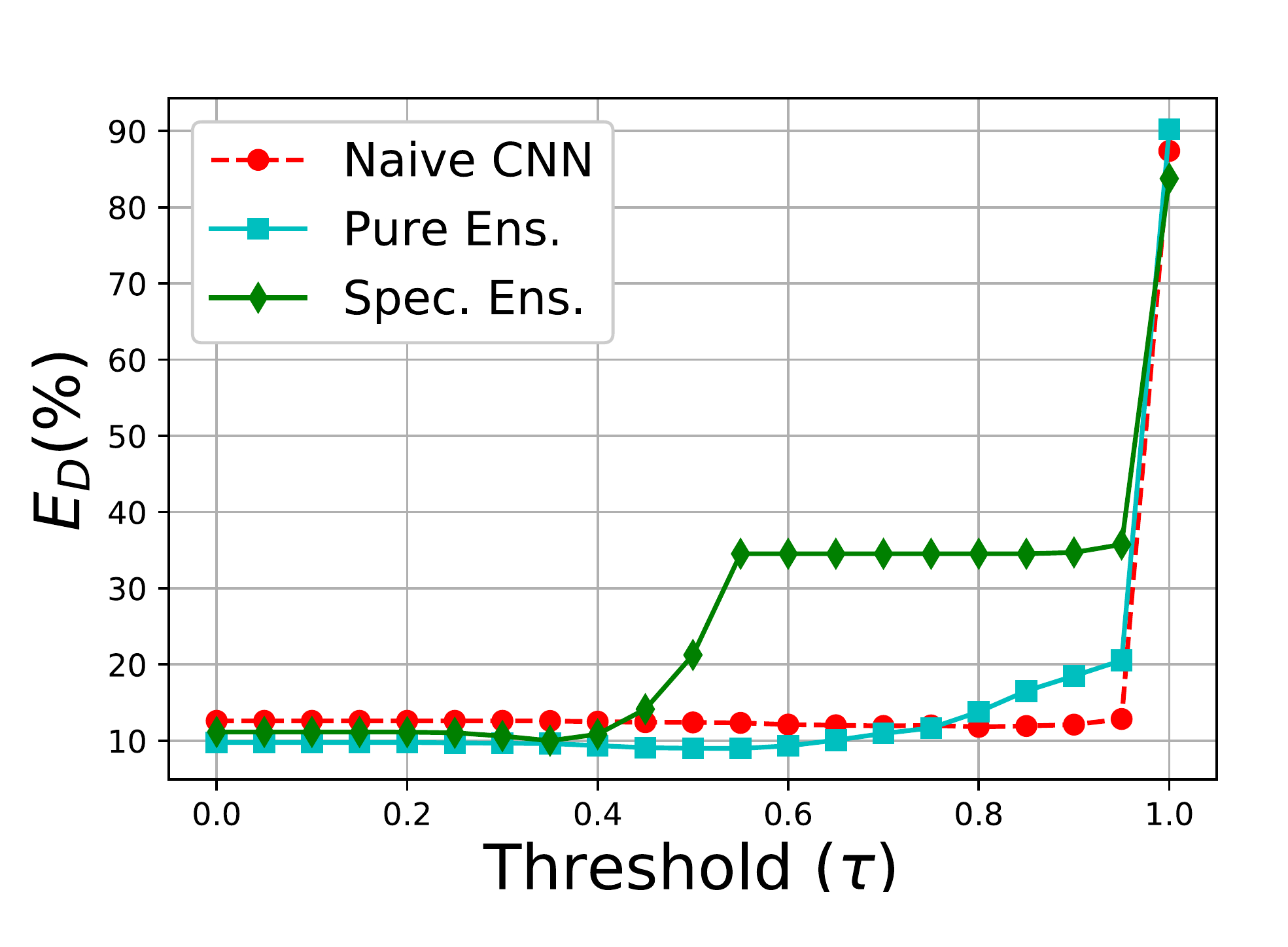}}
\subfigure[CIFAR-10 FGS]{\includegraphics[width = 0.3\textwidth,trim=0cm 7cm 1cm 8cm, clip=true]{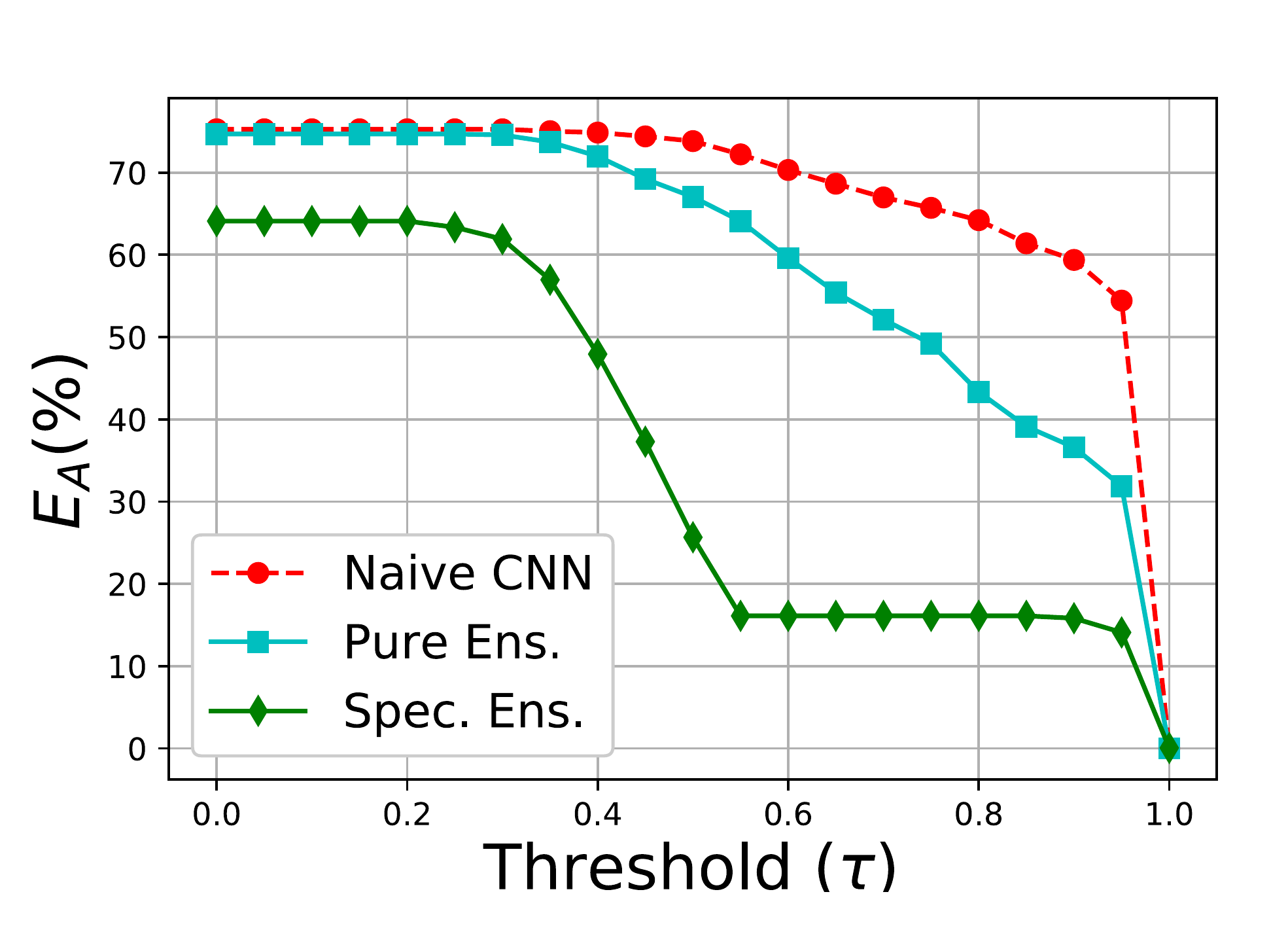}}~
\subfigure[CIFAR-10 TFGS]{\includegraphics[width= 0.3\textwidth,trim=0cm 7cm 1cm 8cm, clip=true]{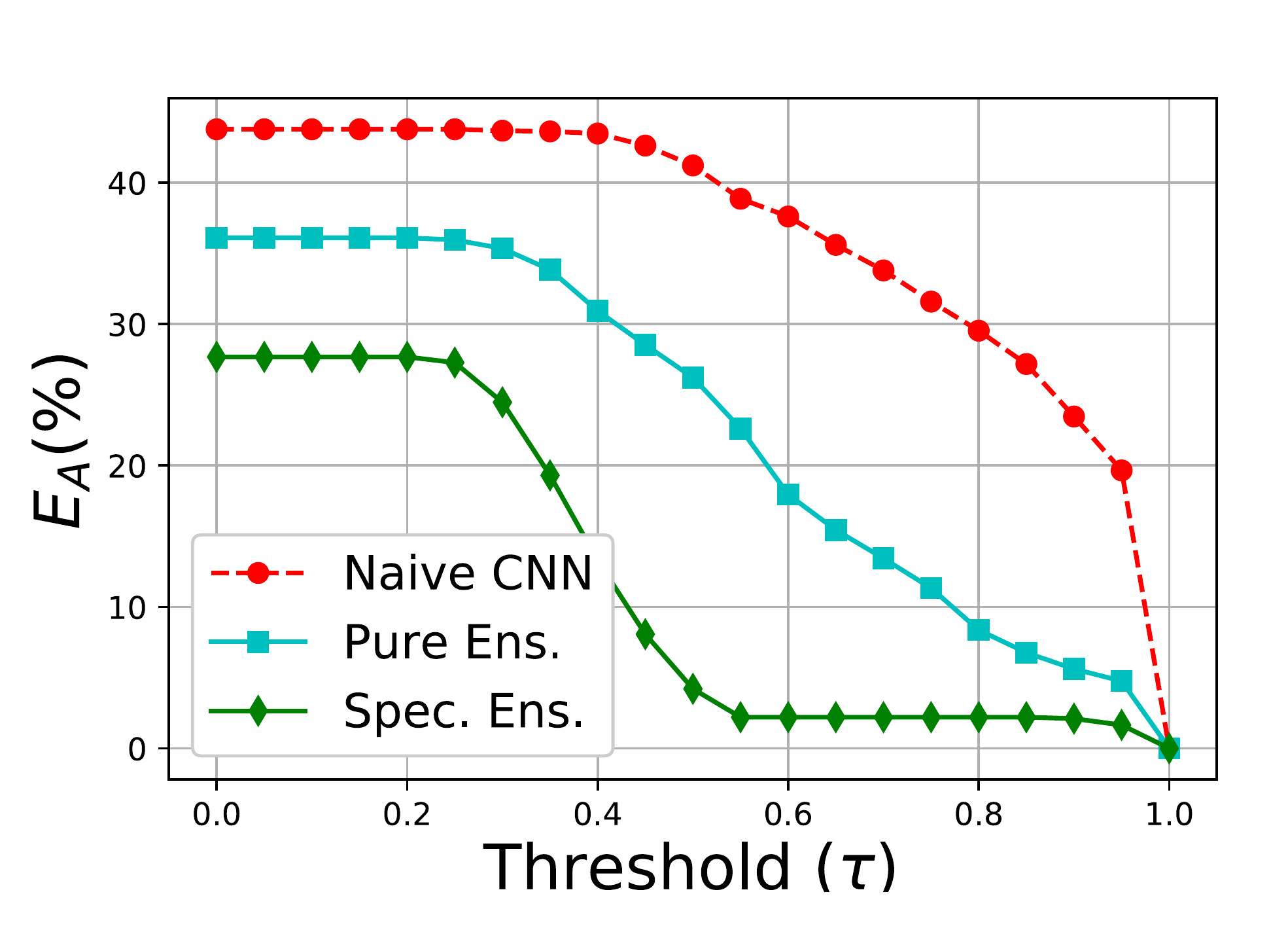}}
\vspace{-1em}
    \caption{The risk rates on the clean test samples and their black-box adversaries as the function of threshold ($\tau$) on the predictive confidence.}
    \label{error_threshold}
\end{figure}

To appropriately compare the methods, we find an optimum threshold that creates small $E_D$ and $E_A$ collectively, i.e.\ $\argmin_{\tau} E_D|\tau+E_A|\tau$. Recall that, as corollary~\ref{coro1} states, in our ensemble of specialists, we can fix the threshold of our ensemble to $\tau^*=0.5$.
In Table~\ref{error_at_optTau}, we compare the risk rates of our ensemble with those of pure ensemble and vanilla CNN at their corresponding optimum thresholds. For MNIST, our ensemble outperforms naive CNN and pure ensemble as it detects a larger portion of MNIST adversaries while its risk rate on the clean samples is only marginally increased. 
Similarly, for CIFAR-10, our approach can detect a significant portion of adversaries at $\tau^*=0.5$, reducing the risk rates on adversaries. However, at this threshold, our approach has higher risk rate on the clean samples than that of two other methods.
\begin{table}[t!]
    \centering
    \begin{tabular}{l|l|cccc}
    \toprule
Task &\multirow{2}{*}{\backslashbox{Methods}{Adversaries}} &   FGS & TFGS & CW & DeepFool \\
&&$E_A$~/~$E_D$&$E_A$~/~$E_D$&$E_A$~/~$E_D$&$E_A$~/~$E_D$\\
\midrule
\multirow{3}{*}{ MNIST }
&  Naive CNN   & 48.21~/~0.84 & 28.15~/~0.84& 41.5~/~0.84 & 88.68~/~0.84 \\
&  Pure Ensemble &  24.02~/~1.1 & 18.35~/~1.1&28.5~/~1.1 & 72.73~/~1.1  \\
& Specialists Ensemble  &  \underline{\textbf{18.58}~/~\textbf{0.73}} & \underline{\textbf{18.05}~/~\textbf{0.73}}& \underline{\textbf{24}~/~\textbf{0.73}} & \underline{\textbf{54.24}~/~\textbf{0.73}} \\

\midrule
\multirow{3}{*}{ CIFAR-10 } 
&  Naive CNN  & 59.37~/~\textbf{12.11} & 23.47~/~\textbf{12.11}& 51.5~/~\textbf{12.11} & 28.81~/~12.11 \\

& Pure Ensemble & 36.59~/~18.5 & \underline{8.37~/~13.79}& \underline{4.0~/~13.79} &\underline{ 7.7~/~18.5}\\
& Specialists Ensemble & \underline{\textbf{25.66}~/~21.25} & ~\textbf{4.21}~/~21.25 & ~\textbf{3.5}~/~21.25 & 6.02~/~21.25
\\
\bottomrule
\end{tabular}
\vspace{1em}
\caption{The risk rate of the clean test set ($E_D|\tau*$) along with that of black-box adversarial examples sets ($E_A|\tau*$) are shown in percentage at the optimum threshold of each method. The methods with the lowest collective risk rate (i.e. $E_A+E_D$) is underlined, while the best results for the two types of risk considered independently are in bold.}
    \label{error_at_optTau}
    \vspace{-2em}
\end{table}
 
\textbf{White-box attacks}:
In the white-box setting, we assume that the attacker has full access to a victim model. Using each method (i.e.\ naive CNN, pure ensemble, and specialists ensemble) as a victim model, we generate different sets of adversaries (i.e.\ FGS, Iterative FGS (I-FGS), and T-FGS). A successful adversarial attack $\mathbf{x}'$ is achieved once the underlying model misclassifies it with a confidence higher than its optimum threshold $\tau^*$. When the confidence for an adversarial example is lower than $\tau*$, it can be easily detected (rejected), thus it is not considered as a successful attack. 

We evaluate the methods by their \emph{white-box attacks success rates}, indicating the number of successful adversaries that satisfies the aforementioned conditions (i.e. a misclassification with a confidence higher than $\tau^*$) during $t$ iterations of the attack algorithm. Table~\ref{white-box-successrate} exhibits the success rates of white-box adversaries (along with their used hyper-parameters) generated by naive CNN ($\tau^*=0.9$), pure ensemble ($\tau^*=0.9$), and specialists ensemble ($\tau^*=0.5$). For the benchmark datasets, the number of iterations of FGS and T-FGS is 2 while that of iterative FGS is 10. As it can be seen in Table~\ref{white-box-successrate}, the success rates of adversarial attacks using ensemble-based methods are smaller than those of naive CNN since diversity in these ensembles hinders generation of adversaries with high confidence.
\begin{table}[t!]
    \centering
\begin{tabular}{l|>{\centering\arraybackslash}p{4em}>{\centering\arraybackslash}p{4em}>{\centering\arraybackslash}p{4em}|>{\centering\arraybackslash}p{4em}>{\centering\arraybackslash}p{4em}>{\centering\arraybackslash}p{4em}}
\toprule
\multicolumn{1}{c|}{}&\multicolumn{3}{c|}{MNIST}&  \multicolumn{3}{c}{CIFAR-10}\\
\hline
\multirow{2}{*}{\backslashbox{Methods}{Adversaries}} & FGS  & T-FGS & I-FGS&  FGS & T-FGS & I-FGS \\
 &$\scriptstyle\epsilon=0.2$ & $\scriptstyle\epsilon=0.2$ &$\scriptstyle\epsilon=2\times10^{-2}$&$\scriptstyle\epsilon=3\times10^{-2}$&$\scriptstyle\epsilon=3\times10^{-2}$&$\scriptstyle\epsilon=3\times10^{-3}$\\
\hline
 Naive CNN& 89.94 & 66.16 & 66.84& 86.16 & 81.38 & 93.93 \\
Pure Ensemble & 71.58 & 50.64 & 48.62& 42.65 & 13.96 & 45.78  \\
 Specialists Ensemble & \textbf{45.15} & \textbf{27.43} & \textbf{13.63}&  \textbf{34.1} & \textbf{7.43} & \textbf{34.20} \\
 \hline
\end{tabular} 

    \caption{Success rate of white-box adversarial examples (lower is better) generated by naive CNN, pure ensemble (5 generalists), and specialists ensemble at their corresponding optimum threshold. An successful white-box adversarial attack should fool the underlying model with a confidence higher than its optimum $\tau^*$.  }
    \label{white-box-successrate}
    \vspace{-2em}
\end{table}

\begin{wrapfigure}{r}{0.4\textwidth}
\begin{minipage}[t]{\linewidth}
\vspace{-2em}
\centering
\includegraphics[width=1\textwidth]{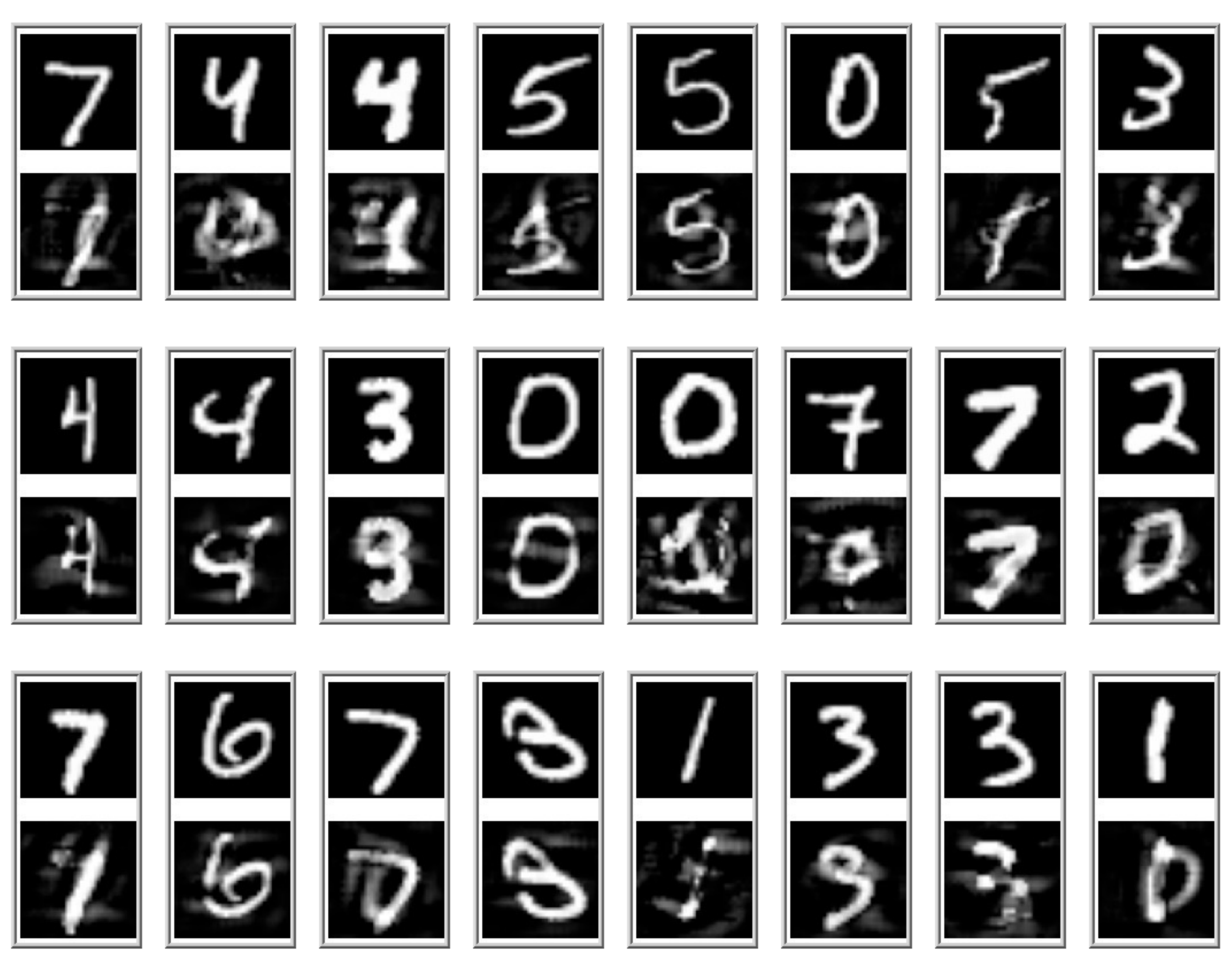}
\label{ensemble-CW}
\vspace{-2em}
\caption{Gray-box CW adversaries that confidently fool our specialists ensemble. According to the definition of adversarial example,however, some of them are not actually adversaries due to the significant visual perturbations.}
 \vspace{-2em}
\end{minipage}
\end{wrapfigure}

\textbf{Gray-box CW attack}: In the gray-box setting, it is often assumed that the attacker is aware of the underlying defense mechanism (e.g. specialists ensemble in our case) but has no access to its parameters and hyper-parameters. Following~\cite{he2017adversarial}, we evaluate our ensemble on CW adversaries generated by another specialists ensemble, composed of 20 specialists and 1 generalist for 100 randomly selected MNIST samples. Evaluation of our specialists ensemble on these targeted gray-box adversaries (called "gray-box CW") reveals that our ensemble provides low confidence predictions (i.e. lower than $0.5$) for $74\%$ of them (thus able to reject them) while $26\%$ have confidence more than $0.5$ (i.e. non-rejected adversaries). Looking closely at those non-rejected adversaries in Fig.~\ref{ensemble-CW}, it can be observed that some of them can even mislead a human observer due to adding very visible perturbation, where the appearance of digits are significantly distorted.

\vspace{-1em}\section{Related Works\label{sec:relatedworks}}

To address the issue of robustness of deep neural networks, one can either \emph{enhance classification accuracy of neural networks to adversaries}, or devise \emph{detectors to identify adversaries} in order to reject to process them. The former class of approaches, known as adversarial training, usually train a model on the training set, which is augmented by adversarial examples. The main difference between many adversarial training approaches lies in the way that the adversaries are created. For example, some \cite{kurakin2016adversarial,huang2015learning,goodfellow2014explaining,madry2017towards} have trained the models with adversaries generated on-the-fly, while others conduct adversarial training with a pre-generated set of adversaries, either produced from an ensemble~\cite{tramer2017ensemble} or from a single model~\cite{rozsa2016adversarial,moosavi2015deepfool}. With the aim detecting adversaries to avoid making wrong decisions over the hostile samples, the second category of approaches propose the detectors, which are usually trained by a training set of adversaries~\cite{metzen2017detecting,feinman2017detecting,grosse2017statistical,Lu_2017_ICCV,meng2017magnet,lee2018simple}.

Notwithstanding the achievement of some favorable results by both categories of approaches, the main concern is that their performances on all types of adversaries are extremely dependent on the capacity of generating an exhaustive set of adversaries, which comprises different types of adversaries. While making such a complete set of adversaries can be computationally expensive, it has been shown that adversely training a model on a specific type of adversaries does not necessarily confer a CNN robustness to other types of adversaries~\cite{tramer2019adversarial,zhang2019limitation}. 

Some ensemble-based approaches~\cite{strauss2017ensemble,kariyappa2019improving} were shown to be effective for mitigating the risk of adversarial examples. Strauss et al.~\cite{strauss2017ensemble} demonstrated some ensembles of CNNs that are created by bagging and different random initializations are less fooled (misclassify adversaries), compared to a single model. Recently, Kariyappa et al.~\cite{kariyappa2019improving} have proposed an ensemble of CNNs, where they explicitly force each pair of CNNs to have dissimilar fooling directions, in order to promoting diversity in the presence of adversaries. However, computing similarity between the fooling directions by each pair of members for every given training sample is computationally expensive, results in increasing training time.

\section{Conclusion}

In this paper, we propose an ensemble of specialists, where each of the specialist classifiers is trained on a different subset of classes. We also devise a simple voting mechanism to efficiently merge the predictions of the ensemble's classifiers. Given the assumption that CNNs are strong classifiers and by leveraging diversity in this ensemble, a gap between predictive confidences of clean samples and those of black-box adversaries is created. Then, using a global fixed threshold, the adversaries predicted with low confidence are rejected (detected). We empirically demonstrate that our ensemble of specialists approach can detect a large portion of black-box adversaries as well as makes the generation of white-box attacks harder. This illustrates the beneficial role of diversity for the creation of ensembles in order to reduce the vulnerability to black-box and white-box adversarial examples.\\[-2.5em]

\subsubsection*{Acknowledgements}
This work was funded by NSERC-Canada, Mitacs, and Prompt-Qu\'ebec. We  thank Annette Schwerdtfeger for proofreading the paper.

\bibliographystyle{splncs}
{\bibliography{ref}}

\end{document}